\documentclass[conference]{IEEEtran}
\pdfoutput=1
\IEEEoverridecommandlockouts
\usepackage[utf8]{inputenc}
\usepackage{amsmath,amssymb,amsfonts,amsthm}
\usepackage{graphicx}
\usepackage{booktabs}
\usepackage{multirow}
\usepackage{hyperref}
\usepackage{xcolor}
\usepackage{subcaption}
\usepackage{tabularx}
\usepackage{pgfplots}
\usepackage{tikz}
\usetikzlibrary{calc}
\usetikzlibrary{plotmarks}
\usepgfplotslibrary{fillbetween}
\pgfplotsset{compat=1.15}
\usepackage{algorithm}
\usepackage{algorithmic}
\floatstyle{ruled}
\restylefloat{algorithm}
\renewcommand{\REQUIRE}{\item[]\textbf{Input:}\;}
\renewcommand{\ENSURE}{\item[]\textbf{Output:}\;}
\newcommand{\STATEX}{\item[]}

\newtheorem{definition}{Definition}
\newtheorem{proposition}{Proposition}
\newtheorem{theorem}{Theorem}

\definecolor{pdtgreen}{RGB}{100,161,27}
\definecolor{pdtlightgreen}{RGB}{149,198,35}
\definecolor{pdtblue}{RGB}{11,93,174}
\definecolor{pdtpurple}{RGB}{106,20,125}
\definecolor{pdtred}{RGB}{145,0,33}
\definecolor{pdtlightblue}{RGB}{59,175,236}

\title{GNN-DIP: Neural Corridor Selection for\\Decomposition-Based Motion Planning}

\author{Peng Xie, Yanliang Huang, Wenyuan Wu, Amr Alanwar
  \thanks{Peng Xie, Yanliang Huang, Wenyuan Wu and Amr Alanwar are with the TUM School of Computation,
    Information and Technology, Department of Computer Engineering,
    Technical University of Munich, 74076 Heilbronn, Germany.
    \texttt{(e-mail: p.xie@tum.de, yanliang.huang@tum.de, wenyuan.wu@tum.de, alanwar@tum.de)}}%
}

\renewcommand{\IEEEtitletopspaceextra}{20pt}

\begin{document}
\maketitle

\begin{abstract}
Motion planning through narrow passages remains a core challenge: sampling-based planners rarely place samples inside these critical regions, and connections between the few samples that land there hug obstacle boundaries and are frequently rejected by collision checking. Decomposition-based planners resolve both issues by partitioning free space into convex cells---every passage is captured exactly as a cell boundary, and any path within a cell is collision-free by construction. However, the number of candidate corridors through the cell graph grows combinatorially with environment complexity, making corridor selection the bottleneck. We present GNN-DIP, a framework that integrates a Graph Neural Network (GNN) with a two-phase Decomposition-Informed Planner (DIP). The GNN predicts portal scores on the cell adjacency graph to bias corridor search toward near-optimal regions while preserving completeness; corridors are then evaluated exactly in 2D (Funnel algorithm on Constrained Delaunay Triangulation) and near-optimally in 3D (portal-face sampling on Slab decomposition). Benchmarks on 2D narrow-passage scenarios, 3D bottleneck environments with up to 246 obstacles, and dynamic 2D settings show that GNN-DIP achieves 99--100\% success rates with 2--280$\times$ speedup over sampling-based baselines.
\end{abstract}

\section{Introduction}
\label{sec:introduction}

Motion planning in cluttered environments is a fundamental problem in robotics. Among its open challenges, narrow passages stand out: in environments with tight doorways, narrow gaps, or double enclosures, the only feasible routes pass through regions occupying a vanishingly small fraction of the space, making their reliable discovery essential for any practical planner.

Sampling-based planners~\cite{zhang2024review}, including RRT*~\cite{karaman2011sampling}, BIT*~\cite{gammell2015batch}, Informed RRT* (iRRT*)~\cite{gammell2014informed}, APT*~\cite{zhang2025apt}, and their variants, provide asymptotic optimality guarantees but face a dual challenge in narrow passages. First, the probability of randomly sampling inside a passage of width~$\varepsilon$ scales as~$\varepsilon^d$, so narrow gaps become exponentially harder to discover in higher dimensions. Second, even when a sample lands inside a narrow passage, the straight-line segments connecting it to neighboring samples run close to obstacle boundaries, complicating collision checking at finite resolution. These challenges compound, so denser sampling or finer collision checking yields diminishing returns.

Decomposition-based planners address both challenges structurally. By partitioning the free space $\mathcal{F}$ into convex cells and constructing a cell adjacency graph, every passage---no matter how narrow---is represented exactly as a cell boundary (portal), so no bottleneck region is missed. Moreover, any path within a convex cell is collision-free by definition, eliminating resolution-dependent collision checking entirely. The continuous planning problem thus reduces to a discrete corridor search followed by local path optimization within each corridor.

The principal limitation of decomposition-based approaches is the corridor selection problem: the number of distinct corridors from start to goal grows combinatorially with the number of cells. In environments with hundreds to thousands of cells, the $k$-shortest path search via Yen's algorithm~\cite{yen1971finding} needs accurate edge weights, yet the default centroid-distance heuristic $w(c_i, c_j) = \|z_i - z_j\|$ correlates poorly with actual path cost for elongated cells, narrow passages, or asymmetric obstacle layouts.

We address the corridor selection problem by training a GNN on the cell adjacency graph to predict portal-level scores. Each portal $p_{ij}$ (the shared boundary between adjacent cells $c_i$ and $c_j$) receives a score $s_{ij} \in [0,1]$ indicating the probability that it lies on a near-optimal corridor. These scores modulate edge weights via $w(c_i, c_j) = d(c_i, c_j) \cdot \exp(-\beta \cdot s_{ij})$, concentrating the $k$-shortest path search on promising regions. Importantly, the modulation is continuous---no corridors are pruned---thereby preserving completeness while improving the quality of initial corridor candidates.

Fig.~\ref{fig:pipeline} illustrates the GNN-DIP ($\mathcal{G}$-DIP) pipeline on a labyrinth with 50 polygon obstacles: CDT produces 385 triangular cells, the GNN isolates a 73-cell corridor containing the near-optimal path (${\sim}81\%$ search-space reduction), parallel Funnel evaluation yields the initial solution in under 20\,ms, and Phase~2 refines it within a shrinking informed ellipsoid.

\begin{figure*}[t]
\centering
\includegraphics[width=0.88\textwidth]{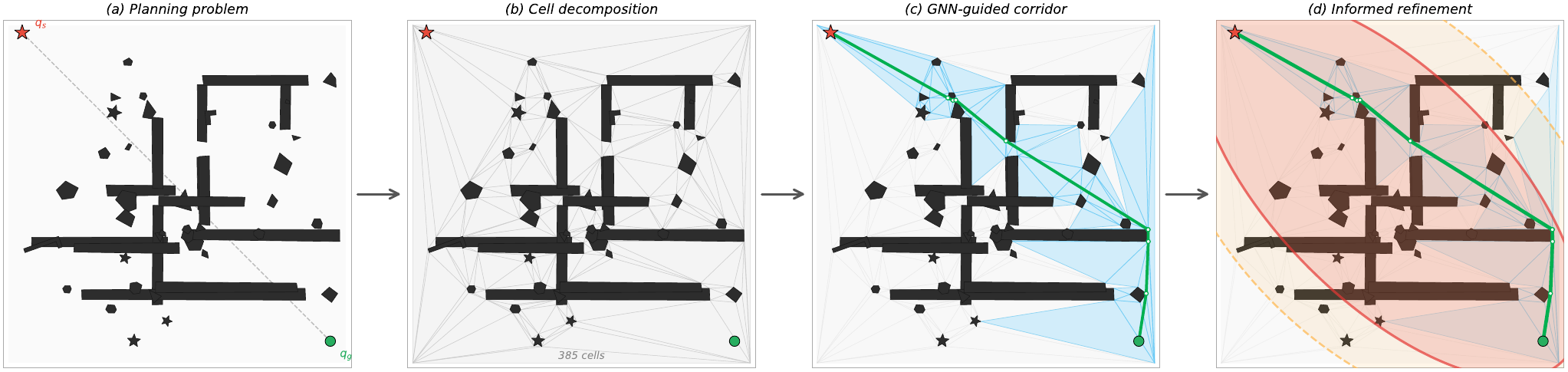}
\caption{GNN-DIP pipeline on a labyrinth with polygon obstacles. (a)~Planning problem with start ($\star$) and goal ($\circ$). (b)~CDT decomposes the free space into 385 triangular cells. (c)~The GNN selects a corridor of 73 cells (light blue), and the Funnel algorithm computes the initial path (green) in parallel across corridor candidates. (d)~Phase~2 refines the solution within a shrinking informed ellipsoid (orange dashed $\to$ red solid).}
\label{fig:pipeline}
\end{figure*}

The contributions of this paper are:
\begin{enumerate}
\item A GNN-based portal scoring framework on cell adjacency graphs, with engineered node/edge features and a training pipeline using focal loss and multi-label supervision (Section~\ref{sec:gnn}).
\item A two-phase Decomposition-Informed Planner (DIP) that combines GNN-guided corridor search with an \emph{adaptation of informed-set pruning}~\cite{gammell2014informed, strub2020pdt} to the cell adjacency graph and Funnel-based corridor evaluation, with proofs of completeness and convergence (Section~\ref{sec:dip}).
\item Comprehensive experiments spanning 2D (310 maps, 18 scenarios), 3D bottleneck environments (4 scenarios, 50 PDT runs each), and dynamic 2D environments (100 planning instances), demonstrating 2--280$\times$ speedup, 99--100\% success rates, and collision safety by construction (Section~\ref{sec:experiments}).
\end{enumerate}

\section{Preliminaries and Related Work}
\label{sec:related_work}

\subsection{Decomposition-Based Motion Planning}

Cell decomposition methods partition the free space into simple regions and plan over the resulting adjacency graph. Exact cell decomposition~\cite{latombe1991robot} constructs cells whose union equals the free space, whereas approximate methods employ regular grids or adaptive subdivisions. Constrained Delaunay Triangulation (CDT)~\cite{shewchuk1996triangle} with parity-based face classification provides a well-studied 2D decomposition with robust implementations in CGAL~\cite{cgal2024}. The Funnel algorithm~\cite{lee1984euclidean} computes the exact Euclidean shortest path through a sequence of adjacent convex polygons in $O(n)$ time, where $n$ is the total number of portal endpoints.

In 3D, exact convex decomposition of general polyhedral free space is computationally expensive. For environments with axis-aligned box (AABB) obstacles, a Slab decomposition exploits obstacle face alignment to produce an exact convex decomposition in near-linear time. Strub and Gammell~\cite{strub2020pdt} introduced AIT* and EIT*, which integrate lazy search and informed sampling. Recent advances include tree-grafting bidirectional planners~\cite{zhang25graft}, estimated informed anytime search with adaptive sampling~\cite{zhang2025TASE}, and genetic programming heuristics for informed trees~\cite{zhang2025git}. The Planner Developer Tools (PDT) framework provides standardized benchmarking infrastructure for OMPL-compatible planners~\cite{sucan2012ompl}.

\subsection{Relation to Graphs of Convex Sets}
\label{sec:gcs}

Closest to our decomposition-plus-adjacency-graph formulation is the Graphs of Convex Sets (GCS) framework~\cite{marcucci2023gcs, marcucci2024shortest}, which also casts planning as a shortest-path problem over a graph whose vertices are convex free-space regions. Both reduce planning to selecting a sequence of convex sets, but differ in how that selection is resolved. GCS embeds the discrete selection and the continuous trajectory into a single mixed-integer convex program solved via a tight convex relaxation, achieving global optimality at the cost of a centralized optimization that scales with the graph. GNN-DIP instead \emph{decouples} the two: a GNN gives learned, millisecond-scale guidance that biases the $k$-shortest corridor search (Sec.~\ref{sec:dip}), and each corridor's path is recovered exactly (Funnel, 2D) or near-exactly (portal-face sampling, 3D). We trade GCS's global-optimality guarantee for far faster initial solutions, anytime refinement, completeness, and high-frequency replanning, while both inherit collision-safety-by-construction from the decomposition. Finally, for the narrow-passage regime we target, exact decomposition offers a further guarantee: every passage appears as a cell boundary, so none is missed. The convex regions underlying GCS, by contrast, are inflated from seed configurations~\cite{deits2015iris} and may fail to cover very thin passages, reintroducing at the region-generation stage the same seeding difficulty that exact decomposition eliminates.

\subsection{Learning and GNNs for Motion Planning}

Neural approaches to motion planning include learned samplers~\cite{ichter2018learning, wang2020neural}, conditional generative models~\cite{ichter2020learned}, and reinforcement learning~\cite{chen2020learning}, all operating in continuous configuration space without exploiting cell decomposition structure. GNNs have been applied to roadmap graphs for collision prediction~\cite{yu2021reducing}, neural planning~\cite{qureshi2021nerp}, and edge cost learning~\cite{zang2023graphmp}. These operate on roadmap graphs (nodes = configurations, edges = local paths), whereas our GNN operates on the cell adjacency graph (nodes = free-space cells, edges = portals)---a fundamentally different and more compact representation. Informed approaches exploit ellipsoidal~\cite{gammell2014informed} or zonotope~\cite{xie2025informed} subsets; our GNN guidance is complementary, biasing corridor search before any solution is found.

\subsection{Problem Formulation}
\label{sec:preliminaries}

This section formalizes the key components of the proposed framework. Throughout, $\mathcal{W} \subseteq \mathbb{R}^d$ ($d \in \{2,3\}$) denotes the workspace.

\begin{definition}[Motion Planning Problem]
\label{def:mpp}
Given a workspace $\mathcal{W}$ with obstacles $\mathcal{O} = \{O_1, \ldots, O_m\}$ (arbitrary simple polygons in 2D, axis-aligned boxes in 3D), the free space is $\mathcal{F} = \mathcal{W} \setminus \bigcup_{i=1}^{m} O_i$. Given start and goal configurations $q_s, q_g \in \mathcal{F}$, the \emph{optimal motion planning problem} seeks a continuous path $\sigma^* : [0,1] \to \mathcal{F}$ with $\sigma(0) = q_s$, $\sigma(1) = q_g$, minimizing the path length $\ell(\sigma) = \int_0^1 \|\dot{\sigma}(t)\| \, dt$:
\begin{equation}
\sigma^* = \arg\min_{\sigma} \; \ell(\sigma).
\end{equation}
\end{definition}

\begin{definition}[Free-Space Decomposition]
\label{def:decomp}
A \emph{free-space decomposition} of $\mathcal{F}$ is a finite collection of closed convex cells $\mathcal{C} = \{c_1, \ldots, c_n\}$ such that (i)~$\bigcup_{i=1}^{n} c_i = \overline{\mathcal{F}}$ (coverage), (ii)~$\mathrm{int}(c_i) \cap \mathrm{int}(c_j) = \emptyset$ for $i \neq j$ (non-overlapping interiors), and (iii)~each $c_i$ is convex.
In 2D, we use Constrained Delaunay Triangulation (CDT) with obstacle edges as constraints and parity-based classification to identify free faces. In 3D with axis-aligned box obstacles, we employ a \emph{Slab convex decomposition}: obstacle face coordinates define axis-aligned splitting planes; obstacle cells are removed and adjacent free cells are greedily merged, producing a compact set of convex boxes.
\end{definition}

\begin{definition}[Cell Adjacency Graph]
\label{def:cag}
The \emph{cell adjacency graph} $G = (\mathcal{C}, \mathcal{P})$ is defined over the free-space cells of Definition~\ref{def:decomp}, with cells as nodes and portals as edges. A \emph{portal} $p_{ij} \in \mathcal{P}$ exists between cells $c_i$ and $c_j$ if they share a $(d{-}1)$-dimensional face (an edge segment in 2D, a rectangular face in 3D). Each portal $p_{ij}$ is characterized by its geometric attributes: endpoints $\{a_{ij}, b_{ij}\}$ and midpoint $m_{ij}$, with size measure $\lambda_{ij}$ defined as the segment length $\|a_{ij} - b_{ij}\|$ in 2D or the face area in 3D.
\end{definition}

\begin{definition}[Corridor]
\label{def:corridor}
A \emph{corridor} $\pi = (c_{i_1}, c_{i_2}, \ldots, c_{i_L})$ is a path in $G$ from the cell containing $q_s$ to the cell containing $q_g$. The corridor defines a connected region $\mathcal{R}_\pi = \bigcup_{\ell=1}^{L} c_{i_\ell}$ through which a collision-free path must pass. The \emph{corridor cost} $\ell(\pi)$ is the length of the shortest path in $\mathcal{F}$ that traverses the cells of $\pi$ in order.
\end{definition}

\section{GNN Portal Scoring}
\label{sec:gnn}

The corridor selection problem can be formulated as an edge classification task on the cell adjacency graph: for each portal $p_{ij}$, predict whether it lies on a near-optimal corridor. A GNN is trained for this task, and its scores bias edge weights in the corridor search.

\subsection{Graph Representation and Feature Engineering}
\label{sec:features}

The graph $G = (\mathcal{C}, \mathcal{P})$ is represented as directed with bidirectional edges (each portal appears twice); each cell $c_i$ carries node features $\mathbf{x}_i \in \mathbb{R}^{d_n}$ and each portal $p_{ij}$ edge features $\mathbf{e}_{ij} \in \mathbb{R}^{d_e}$.

\subsubsection{Node Features ($d_n = 11$ in 2D) and Edge Features ($d_e = 9$ in 2D)}

Table~\ref{tab:features} lists the complete 2D feature set. Node features encode cell geometry (area, aspect ratio $\rho_i = e_{\max}/e_{\min}$), spatial relationships to the query (distances to start, goal, and the start--goal line $d_\perp$), and role indicators. Edge features encode portal geometry, spatial context, and inter-cell relationships. The relative angle $\theta_{ij} = \angle(p_{ij}) - \angle(\overrightarrow{q_s q_g})$ captures alignment between the portal and the global query direction. For the 3D Slab decomposition, features are extended to $d_n{=}14$ and $d_e{=}13$ by adding volumetric cell descriptors (volume, size along each axis, clearance) and 3D portal attributes (face area, portal height, normal axis).

\begin{table}[t]
\centering
\caption{GNN input features: node (cell) and edge (portal).}
\label{tab:features}
{\scriptsize
\setlength{\tabcolsep}{0.4em}
\begin{tabular}{clcl}
\toprule
\multicolumn{2}{c}{\textbf{Node features} ($d_n = 11$)} & \multicolumn{2}{c}{\textbf{Edge features} ($d_e = 9$)} \\
\midrule
$A_i$ & Cell area & $\lambda_{ij}$ & Portal length \\
$z_i^x, z_i^y$ & Centroid coords & $m_{ij}^x, m_{ij}^y$ & Midpoint coords \\
$d(z_i, q_s)$ & Dist.\ to start & $d(m_{ij}, q_s)$ & Dist.\ to start \\
$d(z_i, q_g)$ & Dist.\ to goal & $d(m_{ij}, q_g)$ & Dist.\ to goal \\
$d_\perp(z_i, \overline{q_s q_g})$ & Dist.\ to $sg$-line & $d_\perp(m_{ij}, \overline{q_s q_g})$ & Dist.\ to $sg$-line \\
$\rho_i$ & Aspect ratio & $\theta_{ij}$ & Relative angle \\
$\mathbb{1}[i{=}i_s]$ & Start cell flag & $\|z_i - z_j\|$ & Cell--cell dist. \\
$\mathbb{1}[i{=}i_g]$ & Goal cell flag & $\kappa_{ij}$ & Portal clearance \\
$|\mathcal{N}(i)|$ & Neighbor count & & \\
$\kappa_i$ & Obstacle clearance & & \\
\bottomrule
\end{tabular}
}
\end{table}

\subsection{GNN Architecture}
\label{sec:architecture}

The architecture follows an encode--process--decode pattern. Raw node features are projected to hidden dimension $h=128$ via a linear layer with batch normalization (BN) and ReLU. In 2D, three GCN layers~\cite{kipf2017semi} with symmetric normalization, BN, and dropout ($p{=}0.15$) process the node embeddings; layers 2--3 use two-layer-skip residual connections. For 3D, GCN is replaced by GATv2 (3 layers, 4 attention heads) to better capture the irregular connectivity of Slab cells; the remaining architecture is identical.

For each portal $p_{ij}$, the score is predicted by an MLP ($\mathbb{R}^{2h+d_e} \to h \to 32 \to 1$, ReLU activations) on the concatenation of endpoint embeddings and edge features:
\begin{equation}
\hat{s}_{ij} = \mathrm{sigm}\!\left(\text{MLP}\!\left([\mathbf{h}_i^{(3)} \| \mathbf{h}_j^{(3)} \| \text{BN}_e(\mathbf{e}_{ij})]\right)\right),
\label{eq:edge_score}
\end{equation}
where $\mathrm{sigm}(\cdot) = 1/(1+e^{-x})$ is the sigmoid function (distinguished from path~$\sigma$).

\subsection{Training}
\label{sec:training}

\subsubsection{Label Generation}

Training labels are generated from OMPL baselines (iRRT*, AIT*, BIT*, RRT*) run on each map with a sufficient time budget; portals on near-optimal corridors of the collected paths are labeled positive. Let $\sigma_\text{ref}$ denote the shortest path found across all baseline runs:
\begin{equation}
y_{ij} = \begin{cases} 1 & \exists\, \sigma : \ell(\sigma) \leq (1{+}\epsilon)\,\ell(\sigma_\text{ref}),\; p_{ij} \in \sigma \\ 0 & \text{otherwise} \end{cases}
\end{equation}
where $\epsilon = 0.1$ (10\% suboptimality threshold). Since $\sigma_\text{ref}$ is the best solution found rather than the true optimum $\sigma^*$, label noise may arise when baselines have not converged; in practice, running four planners with a 10\,s budget per map yields near-optimal references. This multi-label scheme assigns positive labels to portals on any near-optimal corridor, capturing the multiplicity of good solutions. The resulting label distribution is severely imbalanced: typically less than 5\% of portals are positive.

\subsubsection{Focal Loss}

To address the extreme class imbalance, focal loss~\cite{lin2017focal} is adopted:
\begin{equation}
\mathcal{L}_\text{FL}(\hat{s}, y) = -\alpha_t (1 - p_t)^\gamma \log(p_t),
\label{eq:focal}
\end{equation}
where $p_t = \hat{s} \cdot y + (1 - \hat{s})(1 - y)$ is the model's estimated probability for the true class, $\alpha_t = \alpha y + (1-\alpha)(1-y)$ balances positive/negative contributions, and $\gamma$ is the focusing parameter. We use $\alpha = 0.85$ and $\gamma = 2.0$. The $(1-p_t)^\gamma$ factor down-weights well-classified negatives, focusing gradient updates on hard positives---the critical portals that the model initially misclassifies.

\subsubsection{Optimization}

Training employs Adam~\cite{kingma2015adam} (initial learning rate $10^{-3}$, weight decay $10^{-4}$) with cosine annealing and early stopping based on validation F1 score (patience 30 epochs). Data are split using stratified sampling (20\% validation ratio). The model ($\sim$150K parameters) trains in under 5 minutes on a single GPU.

\section{Decomposition-Informed Planner}
\label{sec:dip}

The Decomposition-Informed Planner (DIP) operates in two phases on the cell adjacency graph $G$. Phase~1 performs GNN-guided $k$-shortest corridor search with corridor evaluation; Phase~2 refines the search using an informed ellipsoid derived from the current best solution.

\subsection{GNN-Guided Edge Weight Integration}

Given GNN-predicted portal scores $\{\hat{s}_{ij}\}$, we define modified edge weights:
\begin{equation}
w_\text{GNN}(c_i, c_j) = d(c_i, c_j) \cdot \exp(-\beta \cdot \hat{s}_{ij}),
\label{eq:gnn_weight}
\end{equation}
where $d(c_i, c_j) = \|z_i - z_j\|$ is the centroid distance and $\beta > 0$ is a temperature parameter (we use $\beta = 3.0$). \begin{proposition}[Properties of GNN Edge Weights]
\label{prop:weights}
The weight~\eqref{eq:gnn_weight} is strictly positive and continuous; it recovers the centroid-distance baseline when $\hat{s}_{ij} = 0$ (graceful degradation); and higher scores yield lower weights, concentrating the $k$-shortest search on predicted near-optimal portals.
\end{proposition}

\subsection{Phase 1: $k$-Shortest Corridor Search}

Phase~1 applies Yen's algorithm~\cite{yen1971finding} to find $k$ shortest paths in $G$ from the start cell $c_s$ to the goal cell $c_g$, using edge weights $w_\text{GNN}$ when available and centroid distances $d(\cdot, \cdot)$ otherwise. Each corridor is evaluated using the Funnel algorithm (2D) or portal-face sampling with layered-graph DP (3D); the best solution provides an initial cost bound $c_\text{best}^{(1)}$ that seeds Phase~2.

\subsection{Phase 2: Informed Ellipsoid Corridor Refinement}

\begin{definition}[Informed Ellipsoid]
\label{def:ellipsoid}
Given the current best cost $c_\text{best}$, the \emph{informed ellipsoid} is:
\begin{equation}
\mathcal{E}(c_\text{best}) = \{x \in \mathbb{R}^d : \|x - q_s\| + \|x - q_g\| \leq c_\text{best}\}.
\end{equation}
A portal $p_{ij}$ is \emph{informative} if it intersects the ellipsoid, i.e., $p_{ij} \cap \mathcal{E}(c_\text{best}) \neq \emptyset$, or equivalently $\min_{x \in p_{ij}} (\|x - q_s\| + \|x - q_g\|) \leq c_\text{best}$.
\end{definition}

The set in Definition~\ref{def:ellipsoid} is the admissible $L^2$ informed set of Informed RRT*~\cite{gammell2014informed}, and tightening it as $c_\text{best}$ decreases mirrors the informed pruning of AIT*/EIT*~\cite{strub2020pdt}. Our contribution is not the informed set itself but its transfer to the decomposition setting: we prune \emph{portals of the cell adjacency graph} rather than sample configurations directly, and combine it with GNN-scored corridors so refinement concentrates on the corridors the network favors---yielding deterministic anytime behavior on the discrete graph, unlike sampling-based informed methods that must still sample and connect within the set.

Phase~2 iteratively re-runs the $k$-shortest path search on $G$, restricted to portals inside $\mathcal{E}(c_\text{best})$. Only corridors not yet contained in the evaluated set $\mathcal{S}$ are passed to the corridor evaluator (Funnel in 2D, portal-face sampling in 3D). When a better corridor is found, the ellipsoid shrinks accordingly; otherwise, the corridor budget $k'$ is doubled up to $4k$, after which the loop terminates.

\begin{algorithm}[t]
\caption{GNN-DIP: Decomposition-Informed Planner}
\label{alg:dip}
\footnotesize
\begin{algorithmic}[1]
\REQUIRE $\mathcal{W}$, $q_s$, $q_g$, GNN model $f_\theta$, corridor budget $k$, timeout $T$
\ENSURE Best path $\sigma^*$ and cost $c^*$
\STATE Decompose $\mathcal{F}$ into cells; build $G = (\mathcal{C}, \mathcal{P})$
\STATE $\{\hat{s}_{ij}\} \leftarrow f_\theta(G)$
\STATE $w_{ij} \leftarrow d(c_i, c_j) \cdot \exp(-\beta \cdot \hat{s}_{ij})$, ~$\forall\, (c_i, c_j) \in \mathcal{P}$
\STATE $c^* \leftarrow \infty$;~~$\mathcal{S} \leftarrow \emptyset$
\STATE $\Pi_k \leftarrow \text{Yen}(G, c_s, c_g, k, w)$
\FOR{$\pi \in \Pi_k$}
  \STATE $(\ell_\pi, \sigma_\pi) \leftarrow \text{Eval}(\pi, q_s, q_g)$;~~$\mathcal{S} \leftarrow \mathcal{S} \cup \{\pi\}$
  \IF{$\ell_\pi < c^*$}
    \STATE $c^* \leftarrow \ell_\pi$;~~$\sigma^* \leftarrow \sigma_\pi$
  \ENDIF
\ENDFOR
\STATEX
\STATE $k' \leftarrow k$
\WHILE{elapsed time $< T$}
  \STATE $\mathcal{P}' \leftarrow \{p_{ij} \in \mathcal{P} : \min_{x \in p_{ij}} (\|x - q_s\| + \|x - q_g\|) \leq c^*\}$
  \STATE $\Pi' \leftarrow \text{Yen}(G|_{\mathcal{P}'}, c_s, c_g, k', w)$
  \STATE $\text{improved} \leftarrow \text{false}$
  \FOR{$\pi \in \Pi' \setminus \mathcal{S}$}
    \STATE $(\ell_\pi, \sigma_\pi) \leftarrow \text{Eval}(\pi, q_s, q_g)$;~~$\mathcal{S} \leftarrow \mathcal{S} \cup \{\pi\}$
    \IF{$\ell_\pi < c^*$}
      \STATE $c^* \leftarrow \ell_\pi$;~~$\sigma^* \leftarrow \sigma_\pi$;~~$\text{improved} \leftarrow \text{true}$
    \ENDIF
  \ENDFOR
  \IF{$\lnot\,\text{improved}$}
    \STATE $k' \leftarrow \min(2k', 4k)$
    \IF{$k' = 4k$}
      \STATE \textbf{break}
    \ENDIF
  \ENDIF
\ENDWHILE
\RETURN $(\sigma^*, c^*)$
\end{algorithmic}
\end{algorithm}

\subsection{Theoretical Properties}

\begin{theorem}[Completeness of DIP]
\label{thm:complete}
If a collision-free path from $q_s$ to $q_g$ exists in $\mathcal{F}$, and the decomposition $\mathcal{C}$ covers $\mathcal{F}$ with $q_s, q_g$ contained in cells of $G$, then DIP (Algorithm~\ref{alg:dip}) finds a solution path.
\end{theorem}

\begin{proof}[Proof sketch]
GNN weights~\eqref{eq:gnn_weight} preserve graph topology (all edges remain with positive weights), so Yen's algorithm can discover any reachable corridor. For Phase~2, any portal on an optimal corridor satisfies the ellipsoid condition by the triangle inequality, so it is never pruned.
\end{proof}

\begin{theorem}[Convergence of DIP]
\label{thm:converge}
The Phase~2 loop of Algorithm~\ref{alg:dip} terminates in finite iterations. Moreover, each iteration either discovers a strictly better corridor (decreasing $c^*$) or explores no new corridors.
\end{theorem}

\begin{proof}[Proof sketch]
The number of distinct corridors is finite; each is evaluated at most once. Each iteration either discovers a new corridor or triggers termination, so the loop terminates.
\end{proof}

\subsection{Corridor Evaluation}
\label{sec:funnel}

In 2D, the Funnel algorithm~\cite{lee1984euclidean} computes the exact shortest path through a corridor of $L$ convex polygons in $O(L)$ time via string-pulling. In 3D, we employ portal-face sampling: $N_s$ points are sampled uniformly on each portal face, forming a layered DAG from $q_s$ through portal samples to $q_g$. A forward DP sweep finds the shortest path in $O(L \cdot N_s^2)$ time---collision-free by convexity. Adaptive Gaussian re-sampling refines the path for up to $r=3$ iterations.

\subsection{Complexity and System Design}
\label{sec:system}

Decomposition is $O(n \log n)$ (CDT in 2D) or $O(N_x N_y N_z + M)$ (Slab in 3D). Phase~1 runs Yen's $k$-shortest paths~\cite{yen1971finding} in $O(k \cdot |\mathcal{C}| \cdot (|\mathcal{P}| + |\mathcal{C}| \log |\mathcal{C}|))$; each Funnel evaluation is $O(L)$. Phase~2 operates on progressively smaller ellipsoid-filtered subgraphs. GNN inference is $O(L_\text{GNN} \cdot (|\mathcal{P}| \cdot h + |\mathcal{C}| \cdot h^2))$ with $L_\text{GNN}=3$, $h=128$.

The system comprises a C++ planning core ($\sim$5K LOC, OMPL-integrated~\cite{sucan2012ompl}) and a Python GNN module (PyTorch~\cite{paszke2019pytorch} + PyG~\cite{fey2019fast}, $\sim$150K parameters). GNN inference adds 10--50\,ms latency. Default corridor budget: $k=8$ in 2D, $k=16$ ($\mathcal{G}$-DIP) or $k=32$ (unguided DIP) in 3D; 3D portal-face sampling uses $N_s=16$, refinement iterations $r=3$.

\section{Experiments}
\label{sec:experiments}

GNN-DIP is evaluated against unguided DIP and OMPL baselines (best of iRRT*, AIT*, BIT*, EIT*, RRT*) in 2D and 3D, on a single thread of an Intel i7 processor.

\subsection{2D Evaluation}

The 2D benchmark uses 310 polygon maps across 18 scenarios in four complexity tiers by CDT cell count: simple (14--74), medium (80--164), hard (280--672), and very hard (764--2372 cells). With $k{=}8$ and Funnel evaluation, DIP achieves an 89.5\% win rate against OMPL at 10\,ms on simple--hard maps but only 33\% on very hard maps (1000+ cells) due to combinatorial corridor explosion. GNN guidance addresses this: on mega forest (1074 cells), unguided DIP fails while GNN-DIP succeeds with cost 1.295 (vs.\ 1.293 for OMPL); on tight labyrinth (1046 cells), both DIP methods achieve cost 1.737, outperforming OMPL's 1.799.

\paragraph{Decomposition Guarantees Full Reliability on 2D Narrow Passages}
Table~\ref{tab:benchmark} reports PDT~\cite{strub2020pdt} results (100 runs, 2\,s budget) on four very hard 2D scenarios. DIP and $\mathcal{G}$-DIP achieve 100\% success on all scenarios. On Bottleneck, all four sampling-based baselines fall below 3\% success; on Tight Labyrinth, only EIT* reaches 47\%. On Mega Forest, EIT* attains 99\% success but at a median cost of 1.51---17\% higher than $\mathcal{G}$-DIP's 1.29.

\paragraph{GNN Scoring Provides Targeted Speedup on Combinatorially Hard Maps}
$\mathcal{G}$-DIP reduces initial solve time by 4.6$\times$ on Mega Forest (48\,ms vs.\ 223\,ms) and 3.3$\times$ on Bottleneck (158\,ms vs.\ 516\,ms). On Bottleneck, $\mathcal{G}$-DIP also reduces median cost from 1.52 to 1.33 (12.5\%), indicating that GNN-selected corridors are closer to optimal. On Tight Labyrinth and Cluttered Field, where DIP already solves in 16\,ms and 35\,ms, $\mathcal{G}$-DIP matches both cost and latency.

\begin{table}[t]
\centering
\caption{PDT benchmark. Top: SR/median cost. Bottom: time to first solution. 2D: 100 runs, 2\,s; 3D: 50 runs, 20\,s. D-BN = Dense Bottleneck. Bold = best; ``---'' = no solution.}
\label{tab:benchmark}
\resizebox{0.985\columnwidth}{!}{%
\scriptsize
\setlength{\tabcolsep}{0.2em}
\renewcommand{\arraystretch}{0.95}
\begin{tabular}{lcccccc}
\toprule
 & DIP & $\mathcal{G}$-DIP & BIT* & AIT* & iRRT* & EIT* \\
\midrule
\multicolumn{7}{l}{\emph{Success rate / median cost}} \\
\addlinespace[1pt]
MegaFor. & \textbf{1.0}/1.30 & \textbf{1.0}/\textbf{1.29} & --- & --- & .02 & .99/1.51 \\
Bottlnk. & \textbf{1.0}/1.52 & \textbf{1.0}/\textbf{1.33} & .00 & .00 & .00 & .03 \\
TightLab & \textbf{1.0}/\textbf{1.74} & \textbf{1.0}/\textbf{1.74} & --- & .00 & .00 & .47 \\
Clutterd & \textbf{1.0}/\textbf{1.32} & \textbf{1.0}/\textbf{1.32} & --- & --- & .26 & .99/1.34 \\
\addlinespace[2pt]
BN Ofc. & \textbf{1.0}/2.02 & \textbf{1.0}/2.02 & .98/2.18 & .02/2.82 & \textbf{1.0}/2.06 & \textbf{1.0}/\textbf{1.90} \\
BN Maze & \textbf{1.0}/2.19 & \textbf{1.0}/2.19 & \textbf{1.0}/2.40 & .24/2.88 & \textbf{1.0}/2.26 & \textbf{1.0}/\textbf{1.78} \\
BN Lyrs & \textbf{1.0}/2.34 & \textbf{1.0}/2.34 & \textbf{1.0}/2.67 & .00 & .78/2.47 & \textbf{1.0}/\textbf{1.82} \\
\addlinespace[1pt]
D-BN Ofc. & \textbf{1.0}/2.34 & \textbf{1.0}/2.34 & \textbf{1.0}/2.56 & .00 & \textbf{1.0}/2.40 & \textbf{1.0}/\textbf{1.90} \\
\midrule
\multicolumn{7}{l}{\emph{Time to first solution}} \\
\addlinespace[1pt]
MegaFor. & 223\,ms & \textbf{48\,ms} & --- & --- & --- & 1.14\,s \\
Bottlnk. & 516\,ms & \textbf{158\,ms} & --- & --- & --- & --- \\
TightLab & \textbf{16\,ms} & 19\,ms & --- & --- & --- & --- \\
Clutterd & 35\,ms & \textbf{22\,ms} & --- & --- & --- & 508\,ms \\
\addlinespace[2pt]
BN Ofc. & \textbf{18\,ms} & 19\,ms & 0.47\,s & 8.19\,s & 3.24\,s & 0.10\,s \\
BN Maze & 29\,ms & \textbf{19\,ms} & 0.11\,s & 9.75\,s & 1.73\,s & 90\,ms \\
BN Lyrs & 48\,ms & \textbf{38\,ms} & 0.39\,s & --- & 9.77\,s & 73\,ms \\
\addlinespace[1pt]
D-BN Ofc. & 0.26\,s & 0.12\,s & 0.43\,s & --- & 5.86\,s & \textbf{0.17\,s} \\
\bottomrule
\end{tabular}
}%
\end{table}

Fig.~\ref{fig:convergence} shows the convergence plots.

\subsection{3D Bottleneck Benchmark}
\label{sec:complex3d}

To stress-test narrow-passage planning in 3D, we design four bottleneck scenarios where all feasible paths traverse walls with a single narrow door (width 0.035--0.05 in a unit cube)---a regime where the $\varepsilon^3$ hit probability makes sampling-based discovery exponentially hard, while Slab cells capture every door exactly. \emph{Bottleneck Office}: $4{\times}4$ rooms with one narrow door per wall, a floor partition, and 30 clutter boxes (181--190 obstacles); \emph{Bottleneck Maze}: a recursive-division maze with single-door walls, two vertical zones, one floor partition, and 25 clutter boxes (129--175); \emph{Bottleneck Layers}: three layers of $3{\times}3$ rooms with narrow doors and floor holes (radius 0.04--0.05), plus 30 clutter boxes (239--246); \emph{Dense BN Office}: BN Office with 120 extra clutter boxes, yielding $\sim$600 cells and $\sim$1600 portals (vs.\ $\sim$200 originally), where unguided $k$-shortest enumeration becomes the bottleneck.

All six planners are evaluated via PDT (50 runs, 20\,s budget) on the most challenging map per scenario.

\paragraph{DIP Maintains Perfect Reliability Across All 3D Scenarios}
Table~\ref{tab:benchmark} reports all results. DIP and $\mathcal{G}$-DIP maintain 100\% success on all four scenarios, including the dense variant with $\sim$600 cells. AIT* fails entirely on BN Layers and the dense variant (0\%), iRRT* drops to 78\% on BN Layers, and BIT* to 98\% on BN Office. EIT* sustains 100\% across all scenarios but requires the full 20\,s budget.

\paragraph{Speed--Quality Tradeoff Between DIP and Asymptotic Planners}
DIP produces initial solutions in 18--48\,ms, compared to 110--470\,ms for BIT* (4--26$\times$ slower) and 73--100\,ms for EIT*. EIT* achieves lower median costs (1.78--1.90 vs.\ DIP's 2.02--2.34) through asymptotic refinement; DIP trades this for immediate availability. On the dense variant, $\mathcal{G}$-DIP solves in 0.12\,s---2.2$\times$ faster than DIP (0.26\,s) and 3.6$\times$ faster than BIT* (0.43\,s); convergence plots (Fig.~\ref{fig:convergence}) confirm $\mathcal{G}$-DIP converges 2$\times$ faster on this variant.

\paragraph{Neural Corridor Scoring Reduces the Effective Branching Factor}
A $k$-sweep ablation on Dense BN Office (5 maps $\times$ 5 seeds) confirms that $\mathcal{G}$-DIP at $k{=}8$ (cost 2.369; init 64\,ms; total 2011\,ms) matches DIP at $k{=}32$ (2.346; 419\,ms; 6940\,ms): a 1.0\% cost increase buys a 3.5$\times$ total and 6.5$\times$ initialization speedup, showing that GNN scores cut the number of corridors that must be enumerated by ${\sim}4\times$.

\begin{figure*}[!t]
\centering
\begin{subfigure}[t]{0.245\textwidth}
  \centering
  \includegraphics[width=0.96\textwidth]{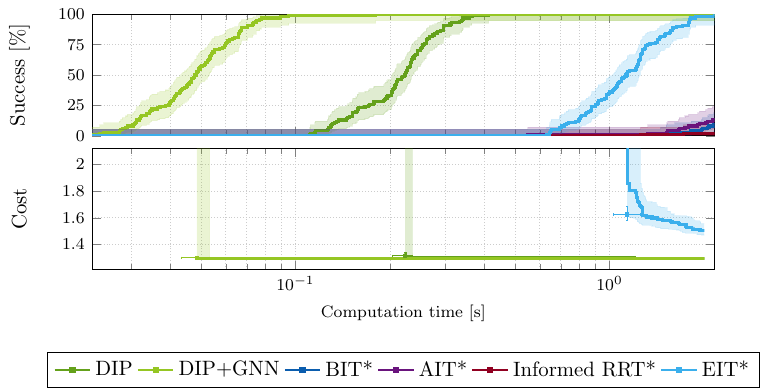}
  \caption{Mega Forest (2D)}
\end{subfigure}\hfill
\begin{subfigure}[t]{0.245\textwidth}
  \centering
  \includegraphics[width=0.96\textwidth]{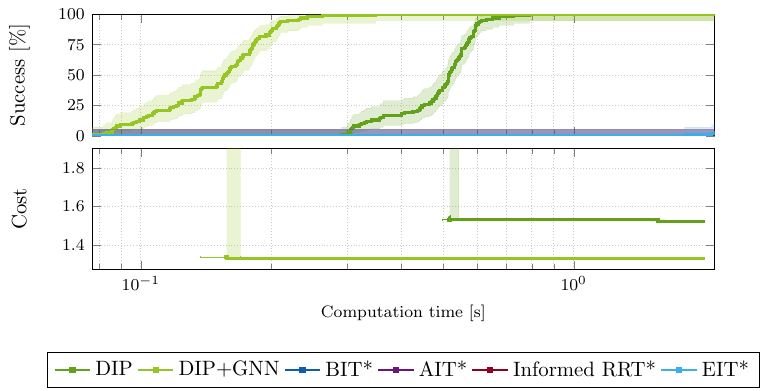}
  \caption{Bottleneck (2D)}
\end{subfigure}\hfill
\begin{subfigure}[t]{0.245\textwidth}
  \centering
  \includegraphics[width=0.96\textwidth]{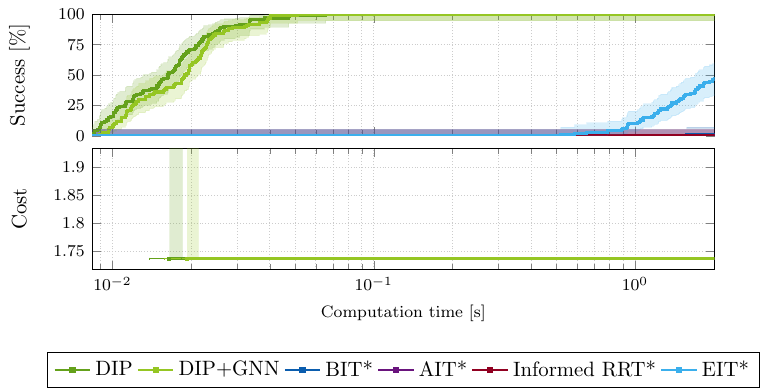}
  \caption{Tight Labyrinth (2D)}
\end{subfigure}\hfill
\begin{subfigure}[t]{0.245\textwidth}
  \centering
  \includegraphics[width=0.96\textwidth]{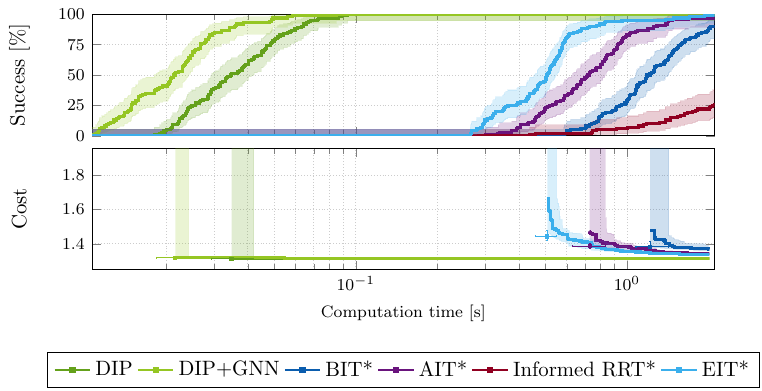}
  \caption{Cluttered Field (2D)}
\end{subfigure}\\[0.5em]
\begin{subfigure}[t]{0.245\textwidth}
  \centering
  \includegraphics[width=0.96\textwidth]{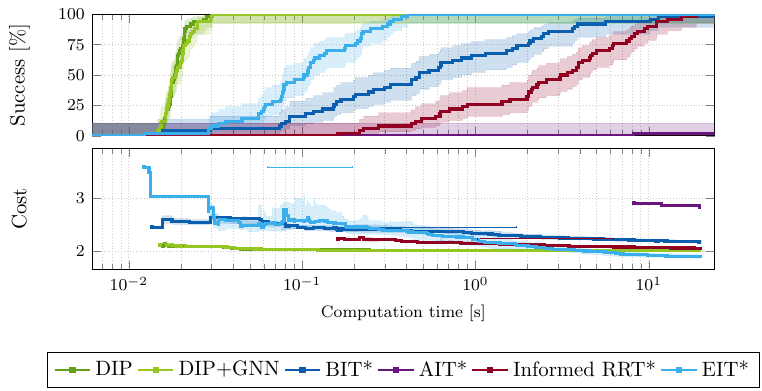}
  \caption{BN Office (3D)}
\end{subfigure}\hfill
\begin{subfigure}[t]{0.245\textwidth}
  \centering
  \includegraphics[width=0.96\textwidth]{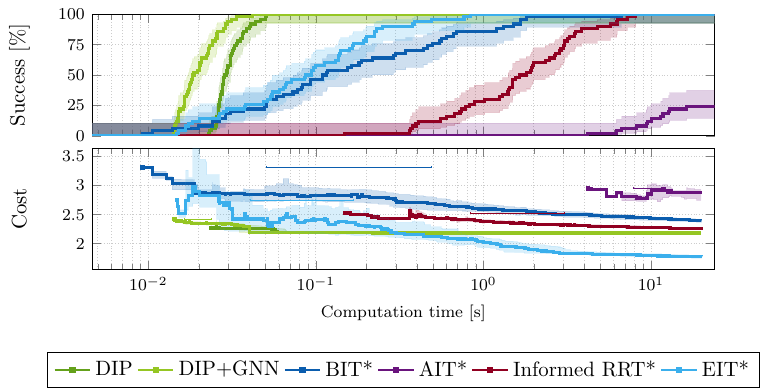}
  \caption{BN Maze (3D)}
\end{subfigure}\hfill
\begin{subfigure}[t]{0.245\textwidth}
  \centering
  \includegraphics[width=0.96\textwidth]{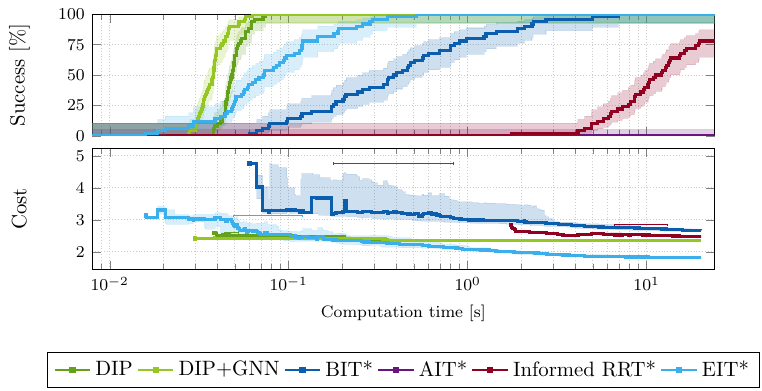}
  \caption{BN Layers (3D)}
\end{subfigure}\hfill
\begin{subfigure}[t]{0.245\textwidth}
  \centering
  \includegraphics[width=0.96\textwidth]{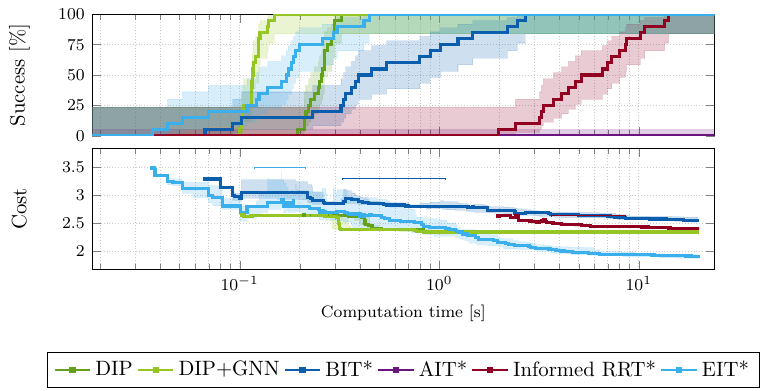}
  \caption{Dense BN Ofc.\ (3D)}
\end{subfigure}
\caption{PDT convergence plots: success rate (top) and median cost (bottom) vs.\ time. \emph{Top row}: 2D very hard scenarios (100 runs, 2\,s). \emph{Bottom row}: 3D bottleneck scenarios (50 runs, 20\,s) and Dense BN Office ($\sim$600 cells).}
\label{fig:convergence}
\end{figure*}

\subsection{Cross-Scenario Generalization}
\label{sec:generalization}

To evaluate whether GNN portal scoring generalizes beyond its training distribution, we conduct two transfer experiments on seven 3D bottleneck scenarios---the four from Sec.~\ref{sec:complex3d} plus Dense BN Maze ($\sim$450 cells) and two individual unseen maps (BN Office \#15, BN Maze \#20). Each is tested with 50 runs and a 20\,s budget.

\paragraph{Leave-One-Type-Out (LOTO)}
Training data spans four scenario families: office, maze, layers, and warehouse (261 samples across 14 subtypes). For each family $f$, we train a LOTO model $\mathcal{G}_{\neg f}$ on all data \emph{excluding} family $f$ and evaluate it on scenarios from $f$.

\paragraph{Simple-to-Complex Transfer}
A model $\mathcal{G}_\text{sim}$ is trained exclusively on four basic scenario types (forest, narrow passage, multi-room, cluttered)---none containing bottleneck structures---and tested on all complex bottleneck scenarios.

\begin{table}[t]
\centering
\caption{Cross-scenario generalization. LOTO$_{\neg f}$: trained excluding the test scenario's family $f$. $\mathcal{G}_\text{sim}$: trained on basic scenarios only. All methods achieve 100\% SR. (50 runs, 20\,s.)}
\label{tab:generalization}
\resizebox{0.985\columnwidth}{!}{%
\scriptsize
\setlength{\tabcolsep}{0.25em}
\renewcommand{\arraystretch}{0.95}
\begin{tabular}{lcccc}
\toprule
 & DIP & $\mathcal{G}$-DIP & LOTO$_{\neg f}$ & $\mathcal{G}_\text{sim}$-DIP \\
\midrule
\multicolumn{5}{l}{\emph{Median cost}} \\
\addlinespace[1pt]
BN Ofc.       & 2.02 & 2.02 & 2.02 & 2.02 \\
BN Maze       & 2.19 & 2.19 & 2.19 & 2.18 \\
BN Lyrs       & 2.34 & 2.34 & 2.34 & 2.34 \\
D-BN Ofc.     & 2.34 & 2.34 & 2.34 & 2.34 \\
D-BN Maze     & \textbf{1.86} & \textbf{1.86} & 2.11 & 2.04 \\
\addlinespace[1pt]
Ofc.\,\#15    & 2.88 & \textbf{2.76} & \textbf{2.76} & \textbf{2.76} \\
Maze\,\#20    & 2.79 & \textbf{2.67} & \textbf{2.67} & 2.67 \\
\midrule
\multicolumn{5}{l}{\emph{Time to first solution}} \\
\addlinespace[1pt]
BN Ofc.       & 14\,ms  & 14\,ms  & \textbf{13\,ms}  & 14\,ms  \\
BN Maze       & 19\,ms  & \textbf{11\,ms}  & 13\,ms  & 12\,ms  \\
BN Lyrs       & 35\,ms  & 26\,ms  & 28\,ms  & \textbf{25\,ms}  \\
D-BN Ofc.     & 185\,ms & 88\,ms  & \textbf{78\,ms}  & 100\,ms \\
D-BN Maze     & 167\,ms & 125\,ms & 129\,ms & \textbf{73\,ms}  \\
\addlinespace[1pt]
Ofc.\,\#15    & 44\,ms  & 26\,ms  & 25\,ms  & \textbf{23\,ms}  \\
Maze\,\#20    & 48\,ms  & \textbf{21\,ms}  & 22\,ms  & 23\,ms  \\
\bottomrule
\end{tabular}
}%
\end{table}

Table~\ref{tab:generalization} shows that LOTO models match the full model within 0.1\% cost in 6 of 7 scenarios with equivalent or faster solve times. The single degradation, Dense BN Maze (+13\% when maze data is excluded), is attributable to maze-specific structure. $\mathcal{G}_\text{sim}$ matches or improves DIP solve times on all bottleneck scenarios (up to 2.3$\times$ speedup on dense variants) despite never encountering bottleneck structures during training, confirming that spatial features (distance, clearance, connectivity) transfer effectively to complex layouts.

\subsection{Dynamic 2D Evaluation}
\label{sec:dynamic2d}

We evaluate GNN-DIP for high-frequency replanning in dynamic 2D environments. Ten scenarios each consist of 10 time steps (100 instances total), with $\sim$50\% static, 30\% moving, and 20\% toggling obstacles (15--58 per step, 96--358 CDT cells). GNN-DIP executes the full pipeline (CDT + GNN + DIP) per step; OMPL runs five planners and selects the best valid result. Both use a 0.5\,s budget. OMPL paths are post-validated via dense collision checking (200 samples/unit); only collision-free paths count as successes.

\paragraph{GNN-DIP Dominates Dynamic Replanning in Reliability, Latency, and Cost}
GNN-DIP achieves 99\% success (99/100; the single failure is genuinely unsolvable) vs.\ OMPL's 40\% after collision post-validation, solving each step in 1.8--44\,ms (50--280$\times$ speedup) including the pipeline overhead (CDT $\sim$10\,ms + GNN $\sim$5\,ms). Fig.~\ref{fig:dynamic_summary} shows consistently lower path costs on all 10 scenarios, with the largest margins (6--8\%) on multi-room environments.

\begin{figure*}[t]
\centering
\includegraphics[width=0.88\textwidth]{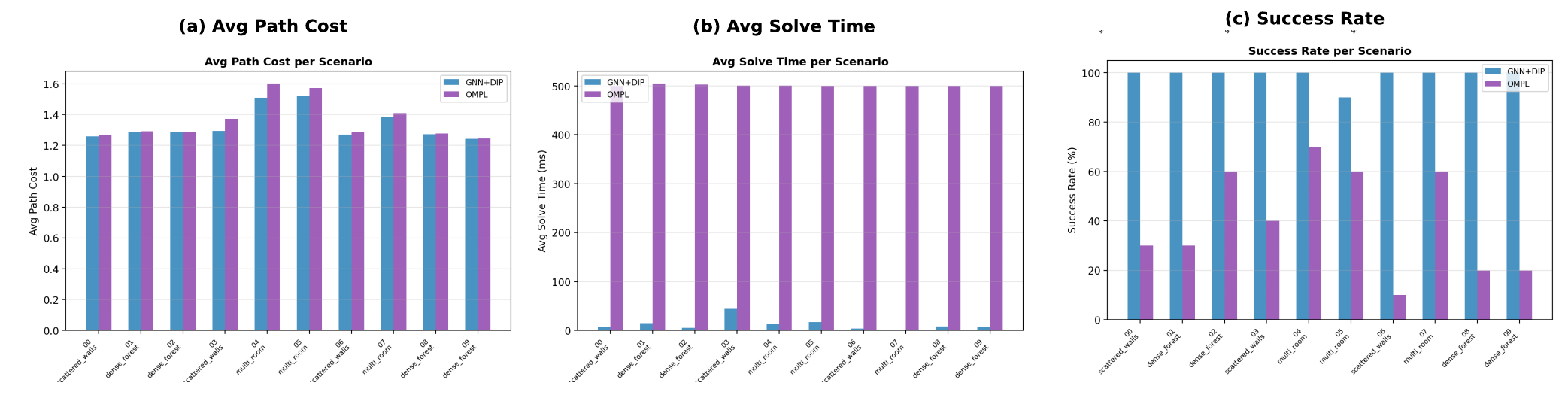}
\caption{Dynamic 2D benchmark across 10 scenarios: (a)~average path cost, (b)~average solve time (GNN-DIP 1.8--44\,ms vs.\ OMPL's 0.5\,s budget), (c)~success rate after collision post-validation (99\% vs.\ 40\%).}
\label{fig:dynamic_summary}
\end{figure*}

\paragraph{Collision Safety by Construction}
The success gap reflects a fundamental architectural difference. To quantify this, we measure OMPL's pre-validation success rate (planner finds \emph{any} path) and post-validation rate (path survives dense collision checking) at two motion-validation resolutions: the default (${\sim}1\%$ of space extent) and $2{\times}$ ($0.5\%$, which doubles the per-edge checking cost). Within the same 0.5\,s budget, pre-validation success is unchanged (489/500 vs.\ 488/500 individual planner runs), confirming that the planning algorithms succeed and the overhead is negligible. However, post-validation success rises from 36\% to 79\%---yet 21\% of steps still contain paths that penetrate thin walls. Raising resolution further would reduce violations at the cost of exploring fewer edges per budget. DIP needs no such tuning: CDT cells partition free space along obstacle boundaries, so corridor paths are collision-free by construction.

\section{Discussion and Conclusion}
\label{sec:discussion}

DIP exploits geometric structure for deterministic, fast initial solutions, while sampling-based planners offer asymptotic optimality at reduced reliability in narrow passages. GNN guidance matters most as complexity grows: the benefit is modest at $\sim$200 cells but reaches a $4\times$ branching-factor reduction at $\sim$600+ cells, and cross-scenario generalization (Table~\ref{tab:generalization}) confirms the learned features are not layout-specific.

\subsection{Extension: CBF-Guarded Execution}
\label{sec:cbf}

DIP produces collision-free \emph{point} paths; for a disk robot of radius~$r$, wall clearance must be enforced at runtime. Rather than inflating obstacles (requiring re-decomposition) or shrinking portals (over-conservative), we define a CBF~\cite{ames2017cbf} on each corridor wall with endpoints $(w_1, w_2)$:
\begin{equation}
h(q) = \operatorname{dist}\bigl((x,y),\; \overline{w_1 w_2}\bigr) - r,
\label{eq:barrier}
\end{equation}
where $\{q : h(q) \geq 0\}$ is the safe region. The corridor structure suits CBF integration for two reasons: (i)~\emph{sparse constraints}---only walls of the current and neighboring corridor cells are monitored (${\leq}\,4$ per step), far fewer than whole-space formulations; (ii)~\emph{mostly-passive monitoring}---in a typical cell the path stays in the convex interior where $h(q) \gg 0$ and the nominal controller runs unmodified (Fig.~\ref{fig:cbf_triangle}a), with intervention only in the few narrow cells where $h \approx r$ (Fig.~\ref{fig:cbf_triangle}b). The filter clamps forward speed to $v \leq \gamma h / |a|$ when heading toward a wall, preserving forward invariance without a QP solver.

\begin{figure}[t]
\centering
\includegraphics[width=\columnwidth]{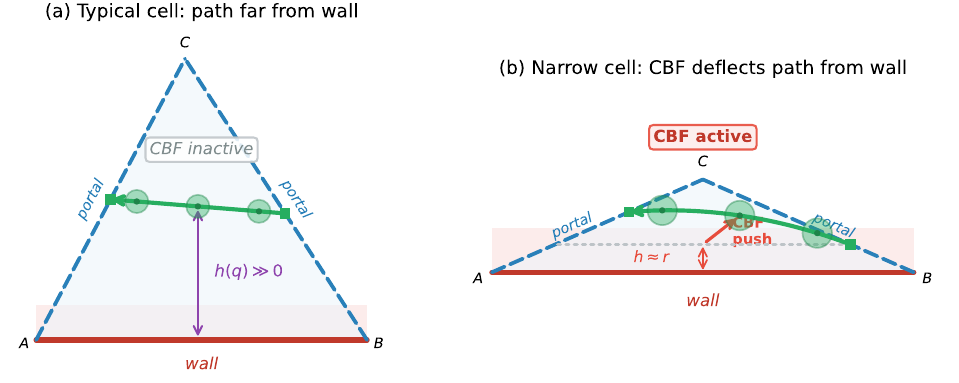}
\caption{CBF behavior within CDT cells. (a)~Typical cell: path traverses the convex interior far from the wall, $h(q) \gg 0$, CBF inactive. (b)~Narrow cell near a bottleneck: path forced close to wall, $h \approx r$, CBF activates to enforce clearance.}
\label{fig:cbf_triangle}
\end{figure}

\subsection{Limitations and Scope}
\label{sec:limitations}

\emph{Workspace vs.\ configuration space.} GNN-DIP operates in the low-dimensional workspace $\mathcal{W}\subseteq\mathbb{R}^d$ ($d\in\{2,3\}$) for point/disk robots, and its guarantees rest on an \emph{exact} convex decomposition of free space. This does not lift directly to the high-dimensional C-spaces of articulated systems (e.g., 6--7-DOF arms): C-space obstacles are curved and non-convex, no exact convex decomposition is known, and the cell count grows rapidly with dimension. The method thus targets low-DOF holonomic and workspace planning. To lift it to configuration space, we plan to follow the Graphs of Convex Sets line of work~\cite{marcucci2023gcs} and replace the exact decomposition with an \emph{approximate convex covering}: we adopt the region-inflation algorithm IRIS~\cite{deits2015iris}---more precisely, its extension to configuration spaces with nonconvex obstacles (IRIS-NP)~\cite{petersen2023irisnp}---to generate a sparse set of large, possibly overlapping convex regions directly in C-space, on which our GNN portal scoring then guides the shortest-path search over the resulting region graph. This trades the constructive completeness of an exact decomposition for scalability, pairing GCS-style regions with learned corridor selection.

\emph{3D obstacle geometry.} Slab decomposition is exact only for axis-aligned box (AABB) obstacles, so general polyhedral, curved, or non-convex geometries are unsupported and the 3D results should be read under this restriction; the 2D CDT pipeline already admits arbitrary simple polygons. A general 3D pipeline could substitute an approximate convex decomposition (e.g., V-HACD~\cite{mamou2009vhacd}), splitting each obstacle into convex pieces on which the same graph, portal features, and GNN scoring apply. The trade-offs are a larger, geometry-dependent cell count, general convex-polygon portals in place of rectangular faces, and loss of exactness where curved boundaries are approximated.

\subsection{Concluding Remarks}

GNN-DIP integrates GNN portal scoring with a two-phase decomposition-informed planner, with formal completeness and convergence guarantees. Across static 2D (310 maps), 3D bottleneck (129--246 obstacles), and dynamic 2D benchmarks, it achieves 99--100\% success with 2--280$\times$ speedups over sampling-based baselines, and GNN guidance matches unguided search quality at a $4\times$ smaller corridor budget.

Future work will address gradient-based 3D path refinement within convex corridors, the C-space and general-geometry extensions outlined in Sec.~\ref{sec:limitations}, C++/ONNX integration for sub-millisecond inference, and experimental validation of the CBF execution layer.



\begin{thebibliography}{30}
\scriptsize

\bibitem{karaman2011sampling}
S.~Karaman and E.~Frazzoli,
``Sampling-based algorithms for optimal motion planning,''
\emph{Int. J. Robot. Res.}, vol.~30, no.~7, pp.~846--894, 2011.

\bibitem{gammell2015batch}
J.~D.~Gammell, S.~S.~Srinivasa, and T.~D.~Barfoot,
``Batch informed trees (BIT*): Sampling-based optimal planning via the heuristically guided search of implicit random geometric graphs,''
in \emph{Proc. IEEE Int. Conf. Robot. Autom. (ICRA)}, 2015, pp.~3067--3074.

\bibitem{gammell2014informed}
J.~D.~Gammell, S.~S.~Srinivasa, and T.~D.~Barfoot,
``Informed RRT*: Optimal sampling-based path planning focused via direct sampling of an admissible ellipsoidal heuristic,''
in \emph{Proc. IEEE/RSJ Int. Conf. Intell. Robots Syst. (IROS)}, 2014, pp.~2997--3004.

\bibitem{yen1971finding}
J.~Y.~Yen,
``Finding the K shortest loopless paths in a network,''
\emph{Management Science}, vol.~17, no.~11, pp.~712--716, 1971.

\bibitem{lee1984euclidean}
D.~T.~Lee and F.~P.~Preparata,
``Euclidean shortest paths in the presence of rectilinear barriers,''
\emph{Networks}, vol.~14, no.~3, pp.~393--410, 1984.


\bibitem{latombe1991robot}
J.-C.~Latombe,
\emph{Robot Motion Planning}, vol.~124 of \emph{The Springer International Series in Engineering and Computer Science}.
Springer Science \& Business Media, 2012.

\bibitem{shewchuk1996triangle}
J.~R.~Shewchuk,
``Triangle: Engineering a 2D quality mesh generator and Delaunay triangulator,''
in \emph{Proc. 1st Workshop Appl. Comput. Geom.}, 1996, pp.~203--222.

\bibitem{cgal2024}
The~CGAL~Project,
\emph{CGAL User and Reference Manual}, 6.0.1~ed.
CGAL Editorial Board, 2024.

\bibitem{strub2020pdt}
M.~P.~Strub and J.~D.~Gammell,
``Adaptively Informed Trees (AIT*) and Effort Informed Trees (EIT*): Asymmetric bidirectional sampling-based path planning,''
\emph{Int. J. Robot. Res.}, vol.~41, no.~4, pp.~390--417, 2022.

\bibitem{sucan2012ompl}
I.~A.~\c{S}ucan, M.~Moll, and L.~E.~Kavraki,
``The Open Motion Planning Library,''
\emph{IEEE Robot. Autom. Mag.}, vol.~19, no.~4, pp.~72--82, 2012.

\bibitem{ichter2018learning}
B.~Ichter, J.~Harrison, and M.~Pavone,
``Learning sampling distributions for robot motion planning,''
in \emph{Proc. IEEE Int. Conf. Robot. Autom. (ICRA)}, 2018, pp.~7087--7094.

\bibitem{wang2020neural}
J.~Wang, W.~Chi, C.~Li, C.~Wang, and M.~Q.-H.~Meng,
``Neural RRT*: Learning-based optimal path planning,''
\emph{IEEE Trans. Autom. Sci. Eng.}, vol.~17, no.~4, pp.~1748--1758, 2020.

\bibitem{chen2020learning}
B.~Chen, B.~Dai, Q.~Lin, G.~Ye, H.~Liu, and L.~Song,
``Learning to plan in high dimensions via neural exploration-exploitation trees,''
in \emph{Proc. Int. Conf. Learn. Representations (ICLR)}, 2020.

\bibitem{ichter2020learned}
B.~Ichter and M.~Pavone,
``Robot motion planning in learned latent spaces,''
\emph{IEEE Robot. Autom. Lett.}, vol.~4, no.~3, pp.~2407--2414, 2019.

\bibitem{yu2021reducing}
C.~Yu and S.~Gao,
``Reducing collision checking for sampling-based motion planning using graph neural networks,''
in \emph{Proc. Adv. Neural Inf. Process. Syst. (NeurIPS)}, vol.~34, 2021, pp.~4274--4289.

\bibitem{qureshi2021nerp}
A.~H.~Qureshi, Y.~Miao, A.~Simeonov, and M.~C.~Yip,
``Motion planning networks: Bridging the gap between learning-based and classical motion planners,''
\emph{IEEE Trans. Robot.}, vol.~37, no.~1, pp.~48--66, 2021.

\bibitem{zang2023graphmp}
X.~Zang, M.~Yin, J.~Xiao, S.~Zonouz, and B.~Yuan,
``GraphMP: Graph neural network-based motion planning with efficient graph search,''
in \emph{Proc. Adv. Neural Inf. Process. Syst. (NeurIPS)}, vol.~36, 2023.

\bibitem{kipf2017semi}
T.~N.~Kipf and M.~Welling,
``Semi-supervised classification with graph convolutional networks,''
in \emph{Proc. Int. Conf. Learn. Representations (ICLR)}, 2017.

\bibitem{lin2017focal}
T.-Y.~Lin, P.~Goyal, R.~Girshick, K.~He, and P.~Doll\'{a}r,
``Focal loss for dense object detection,''
in \emph{Proc. IEEE Int. Conf. Comput. Vision (ICCV)}, 2017, pp.~2999--3007.

\bibitem{kingma2015adam}
D.~P.~Kingma and J.~Ba,
``Adam: A method for stochastic optimization,''
in \emph{Proc. Int. Conf. Learn. Representations (ICLR)}, 2015.

\bibitem{paszke2019pytorch}
A.~Paszke \emph{et al.},
``PyTorch: An imperative style, high-performance deep learning library,''
in \emph{Proc. Adv. Neural Inf. Process. Syst. (NeurIPS)}, 2019, pp.~8024--8035.

\bibitem{fey2019fast}
M.~Fey and J.~E.~Lenssen,
``Fast graph representation learning with PyTorch Geometric,''
in \emph{ICLR Workshop on Representation Learning on Graphs and Manifolds}, 2019.

\bibitem{xie2025informed}
P.~Xie, J.~Betz, and A.~Alanwar,
``Informed hybrid zonotope-based motion planning algorithm,''
\emph{arXiv preprint arXiv:2507.09309}, 2025.

\bibitem{ames2017cbf}
A.~D.~Ames, X.~Xu, J.~W.~Grizzle, and P.~Tabuada,
``Control barrier function based quadratic programs for safety critical systems,''
\emph{IEEE Trans. Autom. Control}, vol.~62, no.~8, pp.~3861--3876, 2017.

\bibitem{zhang2024review}
L.~Zhang, K.~Cai, Z.~Sun, Z.~Bing, C.~Wang, L.~Figueredo, S.~Haddadin, and A.~Knoll,
``Motion planning for robotics: A review for sampling-based planners,''
\emph{Biomimetic Intell. Robot.}, vol.~5, no.~1, p.~100207, 2025.

\bibitem{zhang25graft}
L.~Zhang, Y.~Ling, Z.~Bing, F.~Wu, S.~Haddadin, and A.~Knoll,
``Tree-based grafting approach for bidirectional motion planning with local subsets optimization,''
\emph{IEEE Robot. Autom. Lett.}, vol.~10, no.~6, pp.~5815--5822, 2025.

\bibitem{zhang2025TASE}
L.~Zhang, K.~Cai, Y.~Zhang, Z.~Bing, C.~Wang, F.~Wu, S.~Haddadin, and A.~Knoll,
``Estimated informed anytime search for sampling-based planning via adaptive sampler,''
\emph{IEEE Trans. Autom. Sci. Eng.}, vol.~22, pp.~18580--18593, 2025.

\bibitem{zhang2025git}
L.~Zhang, K.~Cai, Z.~Bing, C.~Wang, and A.~Knoll,
``Genetic informed trees (GIT*): Path planning via reinforced genetic programming heuristics,''
\emph{Biomimetic Intell. Robot.}, vol.~5, no.~3, p.~100237, 2025.

\bibitem{zhang2025apt}
L.~Zhang, S.~Wang, K.~Cai, Z.~Bing, F.~Wu, C.~Wang, S.~Haddadin, and A.~Knoll,
``APT*: Asymptotically optimal motion planning via adaptively prolated elliptical r-nearest neighbors,''
\emph{IEEE Robot. Autom. Lett.}, vol.~10, no.~10, pp.~10242--10249, 2025.

\bibitem{marcucci2023gcs}
T.~Marcucci, M.~Petersen, D.~von Wrangel, and R.~Tedrake,
``Motion planning around obstacles with convex optimization,''
\emph{Sci. Robot.}, vol.~8, no.~84, p.~eadf7843, 2023.

\bibitem{marcucci2024shortest}
T.~Marcucci, J.~Umenberger, P.~A.~Parrilo, and R.~Tedrake,
``Shortest paths in graphs of convex sets,''
\emph{SIAM J. Optim.}, vol.~34, no.~1, pp.~507--532, 2024.

\bibitem{mamou2009vhacd}
K.~Mamou and F.~Ghorbani,
``A simple and efficient approach for 3D mesh approximate convex decomposition,''
in \emph{Proc. IEEE Int. Conf. Image Process. (ICIP)}, 2009, pp.~3501--3504.

\bibitem{deits2015iris}
R.~Deits and R.~Tedrake,
``Computing large convex regions of obstacle-free space through semidefinite programming,''
in \emph{Algorithmic Foundations of Robotics XI (WAFR)}. Springer, 2015, pp.~109--124.

\bibitem{petersen2023irisnp}
M.~Petersen and R.~Tedrake,
``Growing convex collision-free regions in configuration space using nonlinear programming,''
\emph{arXiv preprint arXiv:2303.14737}, 2023.

\end{thebibliography}
\end{document}